%% file: main.tex
\documentclass[runningheads]{llncs}

\usepackage[T1]{fontenc}
\usepackage{xspace}
\usepackage{xcolor}
\usepackage[hidelinks]{hyperref}
\usepackage{amsmath, amssymb}
\usepackage{tikz}
\usepackage{stmaryrd}
\usepackage{booktabs}
\usepackage{listings}
\usepackage{orcidlink}

\input{conditional-compilation}

\papermode{}  

\input{macros}

\begin{document}

\title{{\ESPARQL}: Representing and Reconciling Agnostic and Atheistic Beliefs \\ in RDF-star Knowledge Graphs}
\titlerunning{{\ESPARQL}: Representing and Reconciling Agnostic and Atheistic Beliefs}

\author{
  Xinyi Pan\inst{1}\orcidlink{0009-0002-7082-8156} \and%
  Daniel Hernandez\inst{1}\orcidlink{0000-0002-7896-0875} \and%
  Philipp Seifer\inst{2}\orcidlink{0000-0002-7421-2060} \and%
  Ralf Lämmel\inst{2}\orcidlink{0000-0001-9946-4363} \and%
  Steffen Staab\inst{1,3}\orcidlink{0000-0002-0780-4154}%
}
\authorrunning{
  Pan, Hernandez, Seifer, Lämmel, and Staab
}

\institute{
  Institute for Artifical Intelligence, University of Stuttgart, Germany \\
  \email{xinyi.pan@ms.informatik.uni-stuttgart.de} \\
  \email{\{daniel.hernandez, steffen.staab\}@ki.uni-stuttgart.de}
  \and%
  The Software Languages Team, University of Koblenz, Germany \\
  \email{\{pseifer, laemmel\}@uni-koblenz.de}
  \and%
  Web and Internet Science Research Group, University of Southampton, United Kingdom
}

\maketitle

\begin{abstract}
  Over the past few years, we have seen the emergence of large knowledge graphs combining information from multiple sources. Sometimes, this information is provided in the form of assertions about other assertions, defining contexts where assertions are valid. A recent extension to RDF which admits statements over statements, called {\RDFStar}, is in revision to become a W3C standard. However, there is no proposal for a semantics of these  {\RDFStar} statements nor a built-in facility to operate over them. In this paper, we propose a query language for epistemic {\RDFStar} metadata based on a four-valued logic, called {\ESPARQL}. Our proposed query language extends {\SPARQLStar}, the query language for {\RDFStar}, with a new type of \texttt{FROM} clause to facilitate operating with multiple and sometimes conflicting beliefs. We show that the proposed query language can express four use case queries, including the following features: (i) querying the belief of an individual, (ii) the aggregating of beliefs, (iii) querying who is conflicting with somebody, and (iv) beliefs about beliefs (i.e., nesting of beliefs).

\end{abstract}

\input{1-Introduction}

\input{3-UseCase}

\input{2-Preliminaries}

\input{4-AnnotatedSPARQLStar}

\input{5-eSPARQL}

\input{6-UseCaseDiscussion}

\input{7-Safe_eSPARQL}

\input{8-RelatedWork}

\input{9-Conclusion}

\subsection*{Implementation.}
We implemented eSPARQL on top of Apache Jena~\cite{apache-jena}, an open source SPARQL-star engine. Our implementation~\cite{darus-4344_2024} is available under a free software license.

\subsection*{Acknowledgements}

This work was funded by the Deutsche Forschungsgemeinschaft (DFG, German Research Foundation) under the DFG Germany's Excellence Strategy -- EXC 2120/1 -- 390831618, and the DFG Excellence Strategy -- EXC 2075 -- 390740016. We acknowledge the support by the Stuttgart Center for Simulation Science (SimTech).

\bibliographystyle{splncs04}
\bibliography{main}


\end{document}

%% file: conditional-compilation.tex
\usepackage{etoolbox}

\newtoggle{REPORT}
\newtoggle{DEBUG}
\toggletrue{REPORT}
\toggletrue{DEBUG}

\definecolor{reportcolor}{HTML}{188781}
\definecolor{papercolor}{HTML}{536872}

\newcommand{\papermode}{\togglefalse{REPORT}\togglefalse{DEBUG}}

%% file: macros.tex
\usepackage{enumitem}

\newcommand{\ESPARQL}{eSPARQL}
\newcommand{\SPARQL}{SPARQL}
\newcommand{\RDF}{RDF}
\newcommand{\SPARQLStar}{SPARQL-star}
\newcommand{\RDFStar}{RDF-star}

\usepackage{mdframed}
\newmdenv[%
skipabove=0.5em,
skipbelow=0.5em,
outerlinewidth=1,
backgroundcolor=yellow!30
]{Daniel}

\newcommand{\FOUR}{\mathcal{FOUR}}
\newcommand{\FOURInf}{\mathcal{FOUR}_{\mathrm{in}}}
\newcommand{\FOURTruth}{\mathcal{FOUR}_{\mathrm{th}}}

\newcommand{\true}{{\top}}
\newcommand{\false}{{\bot}}
\newcommand{\unknown}{{\vdash}}
\newcommand{\conflicting}{{\dashv}}

\newcommand{\leqTruth}{\mathrel{\leq_{\operatorname{tr}}}}
\newcommand{\leqInf}{\mathrel{\leq_{\operatorname{in}}}}
\newcommand{\LeqTruth}[2]{{#1}\leqTruth{#2}}
\newcommand{\LeqInf}[2]{{#1}\leqInf{#2}}

\newcommand{\TruthMeet}{\owedge}
\newcommand{\TruthJoin}{\ovee}
\newcommand{\InfMeet}{\ogreaterthan}
\newcommand{\InfJoin}{\olessthan}

\newcommand{\Reduce}[2]{{#1}(#2)}

\newcommand{\Id}[1]{{\operatorname{Id}(#1)}}
\newcommand{\Absorbing}[1]{\operatorname{Abb}(#1)}

\newcommand{\AtomicBeliefQuery}[3]{[{#1},{#2},{#3}]}
\newcommand{\GraphContext}[2]{{#1|_{#2}}}

\newcommand{\Structure}[1]{\mathcal{#1}}
\newcommand{\sM}{\Structure{M}}
\newcommand{\sK}{\Structure{K}}
\newcommand{\SumStructure}[1]{+_{\Structure{#1}}}
\newcommand{\ProdStructure}[1]{\times_{\Structure{#1}}}
\newcommand{\SumM}{\SumStructure{M}}
\newcommand{\SumK}{\SumStructure{K}}
\newcommand{\ProdK}{\ProdStructure{K}}
\newcommand{\ZeroStructure}[1]{0_{\Structure{#1}}}
\newcommand{\OneStructure}[1]{1_{\Structure{#1}}}
\newcommand{\ZeroM}{\ZeroStructure{M}}
\newcommand{\ZeroK}{\ZeroStructure{K}}
\newcommand{\OneK}{\OneStructure{K}}

\newcommand{\SetIRIs}{\mathbf{I}}
\newcommand{\SetVars}{\mathbf{V}}

\newcommand{\AGraph}{G}
\newcommand{\Mappings}[1]{\Omega|_{#1}}

\newcommand{\variable}[1]{\texttt{{?}#1}}
\newcommand{\varX}{\variable{x}}
\newcommand{\varY}{\variable{y}}
\newcommand{\varDeity}{\variable{deity}}

\newcommand{\Individual}[1]{\mathsf{#1}}

\newcommand{\Jesus}{\Individual{Jesus}}
\newcommand{\Zeus}{\Individual{Zeus}}

\newcommand{\Pope}{\Individual{PopeDI}}
\newcommand{\Arius}{\Individual{Arius}}
\newcommand{\Christianity}{\Individual{Christianity}}
\newcommand{\Russell}{\Individual{Russell}}
\newcommand{\type}{\Individual{a}}
\newcommand{\FullDeity}{\Individual{FullDeity}}
\newcommand{\Christian}{\Individual{Christian}}

\newcommand{\believesToBeTrue}{\mathsf{believesToBeTrue}}
\newcommand{\believesToBeFalse}{\mathsf{believesToBeFalse}}
\newcommand{\believesToBeUnknown}{\mathsf{believesToBeUnknown}}
\newcommand{\believesToBeConflicted}{\mathsf{believesToBeConflicted}}

\newcommand{\triple}[3]{(#1, #2, #3)}
\newcommand{\ToBeTrue}[2]{\triple{#1}{\believesToBeTrue}{#2}}
\newcommand{\ToBeFalse}[2]{\triple{#1}{\believesToBeFalse}{#2}}
\newcommand{\ToBeUnknown}[2]{\triple{#1}{\believesToBeUnknown}{#2}}
\newcommand{\ToBeConflicted}[2]{\triple{#1}{\believesToBeConflicted}{#2}}

\newcommand{\aAND}{\mathbin{\mathsf{AND}}}
\newcommand{\aUNION}{\mathbin{\mathsf{UNION}}}

\newcommand{\aSELECT}{\operatorname{\mathsf{SELECT}}}
\newcommand{\aFILTER}{\mathbin{\mathsf{FILTER}}}

\newcommand{\qAND}[2]{(#1 \aAND #2)}
\newcommand{\qUNION}[2]{(#1 \aUNION #2)}
\newcommand{\qFILTER}[2]{(#1 \aFILTER #2)}
\newcommand{\qSELECT}[2]{(\aSELECT\,#1\,#2)}

\newcommand{\aFILTERe}[1]{\mathbin{\mathsf{FILTER}_{#1}}}
\newcommand{\aSELECTe}[1]{\mathsf{SELECT}_{#1}}
\newcommand{\aBELIEFe}{\mathsf{BELIEF}}

\newcommand{\aANDTruth}{\TruthMeet}
\newcommand{\aANDInf}{\InfMeet}
\newcommand{\aUNIONTruth}{\TruthJoin}
\newcommand{\aUNIONInf}{\InfJoin}

\newcommand{\aSTATEUPDATE}{\mathbin{\mathsf{MAP}}}

\newcommand{\qFILTERe}[3]{(#2 \aFILTERe{#1} #3)}
\newcommand{\qSELECTe}[3]{(\aSELECTe{#1}\;{#2}\;{#3})}
\newcommand{\qBELIEFe}[2]{(\aBELIEFe\;{#1}\;{#2})}
\newcommand{\qSTATEUPDATE}[4]{(#1 \aSTATEUPDATE (#2 ,#3, #4))}

\newcommand{\qANDTruth}[2]{(#1 \aANDTruth #2)}
\newcommand{\qANDInf}[2]{(#1 \aANDInf #2)}
\newcommand{\qUNIONTruth}[2]{(#1 \aUNIONTruth #2)}
\newcommand{\qUNIONInf}[2]{(#1 \aUNIONInf #2)}

\newcommand{\qBound}[1]{\operatorname{bound}(#1)}
\newcommand{\qState}[1]{\operatorname{state}(#1)}

\newcommand{\True}{\mathsf{true}}
\newcommand{\False}{\mathsf{false}}
\newcommand{\Error}{\mathsf{error}}

\newcommand{\kset}[2]{{\llparenthesis{#1}\rrparenthesis}_{#2}}
\newcommand{\sparqlEval}[2]{{\llbracket{#1}\rrbracket}_{#2}}

\newcommand{\st}{\;:\;}

\newcommand{\dom}{\operatorname{dom}}
\newcommand{\var}{\operatorname{var}}
\newcommand{\support}{\operatorname{supp}}
\newcommand{\inScope}{\operatorname{inScope}}

\makeatletter
\newcommand\varitem[1]{\item[{U\arabic{enumi}\rlap{$#1$}}]%
  \edef\@currentlabel{U\arabic{enumi}{$#1$}}}
\makeatother

\definecolor{codegray}{rgb}{0.42,0.42,0.42}
\definecolor{backcolour}{rgb}{0.97,0.97,0.97}
\lstdefinestyle{mystyle}{
    backgroundcolor=\color{backcolour},   
    keywordstyle=\color{codegray},
    numberstyle=\tiny\color{codegray},
    basicstyle=\ttfamily\scriptsize,
    breakatwhitespace=false,         
    breaklines=true,                 
    captionpos=b,                    
    keepspaces=true,                 
    numbers=left,                    
    numbersep=5pt,                  
    showspaces=false,                
    showstringspaces=false,
    showtabs=false,                  
    tabsize=2,
}
\lstdefinelanguage{SPARQL}{
    keywords = {SELECT,WHERE,FILTER,INFO,FROM,BELIEF,a,MAP,IF,TO,ELSE,STATE,IS}
}

%% file: 1-Introduction.tex
\section{Introduction}%
\label{sec:introduction}

Over the past few years, we have seen the emergence of large knowledge graphs combining information from multiple sources, such as Wikidata~\cite{DBLP:conf/www/Vrandecic12}, YAGO~\cite{DBLP:conf/esws/TanonWS20}, and DBpedia~\cite{DBLP:conf/semweb/AuerBKLCI07}. These knowledge graphs provide information about a great variety of entities, such as people, countries, universities, as well as facts over these entities, which are codified as triples. For example, the fact that Albert Einstein was born in Germany can be encoded as a triple $\triple{\Individual{Einstein}}{\Individual{wasBorn}}{\Individual{Germany}}$. Triples like this are the information units of the Resource Description Framework (RDF)~\cite{rdf11}, a data model for representing information about World Wide Web resources, and SPARQL~\cite{SPARQL2013}, the query language proposed by the W3C to query RDF data.

RDF and SPARQL have been built on the simplifying assumption that statements (triples) are either \emph{true} or \emph{unknown}, leading to a representation of knowledge that can never conflict with each other, and thus can never lead to inconsistencies. To present competing statements, knowledge bases include reified statements without introducing them as a stated statement. For example, the statement ``Jesus is a deity'' is provided in a reified form in Wikidata, and annotated as ``supported by Christianity, Messianic Judaism, and Manichaeism'' but ``disputed by Islam, Atheism, and Judaism''. Acknowledging the need to make statements about statements, two recent extensions, namely RDF-star and SPARQL-star~\cite{rdfstar}, were proposed. However, these extensions do not deviate from the assumption that statements are either true or unknown.

To illustrate the need to work with statements that are \emph{true}, \emph{unknown}, \emph{false}, or \emph{conflicted}, let us consider some disagreements on beliefs in the Christian religion. Several councils were convened to establish consensus about different aspects of the nature of Jesus, his birth, his mother's birth, his father, and his existence before his birth. By stating a dogma, a council established what should be considered true and what is a heresy.
For example, one of these beliefs can be encoded in RDF-star as a triple where the last element is also a triple:
\[\ToBeTrue{\Pope}{\triple{\Jesus}{\type}{\FullDeity}}.\]

RDF-star does not provide a semantics for this statement, but we can assume that it means ``Pope Damasus the First believes that Jesus is a full Deity.'' Indeed, the Pope presided over the Council of Nice, where this statement became a dogma. This belief was not shared by all participants of the Council. Arius believed that Jesus was created by decision of God, and thus his deity is not as full as the deity of his father. Arius' belief can be encoded as follows:
\[\ToBeFalse{\Arius}{\triple{\Jesus}{\type}{\FullDeity}}.\]
The two statements above do not entail that ``Jesus is a full deity'' nor the contrary, but ascribe some different people's beliefs. Since we have contradictory opinions between Christians, we can say that this statement is conflicted according to Christianity:
\[\ToBeConflicted{\Christianity}{\triple{\Jesus}{\type}{\FullDeity}}.\]

On the other hand, the mathematician Bertrand Russell described his religious beliefs as follows: \emph{``In regard to the Olympic gods, speaking to a purely philosophical audience, I would say that I am an Agnostic. But speaking popularly, I think that all of us would say in regard to those gods that we were Atheists. In regard to the Christian God, I should, I think, take exactly the same line.''} Regarding the aforementioned statement, Russell's beliefs can be encoded as follows:
\[
  \begin{aligned}[t]
    &\ToBeUnknown{\Russell}{\triple{\Jesus}{\type}{\FullDeity}},\\
    &\ToBeUnknown{\Russell}{\triple{\Zeus}{\type}{\FullDeity}}.\\
  \end{aligned}
\]
Russell defines himself as agnostic, since he cannot decide whether the statement is false or true, and he also describes his beliefs about the beliefs of Christians who may be believers regarding Jesus, but atheists regarding an Olympic deity. This Russell's argument applied to $\Pope$ is encoded as follows:
\[
  \begin{aligned}[t]
    &\ToBeTrue{\Russell}{\ToBeTrue{\Pope}{\triple{\Jesus}{\type}{\FullDeity}}},\\
    &\ToBeTrue{\Russell}{\ToBeFalse{\Pope}{\triple{\Zeus}{\type}{\FullDeity}}}.\\
  \end{aligned}
\]

In order to deal with multiple views of the world, it becomes necessary to operate on sets of facts. For example, we would like to execute a SPARQL query against the set of all facts that are valid according to Christianity, or a given person. However, SPARQL does hardly support working with sets of facts. Indeed, SPARQL has an ambivalent attitude towards sets of facts, as it allows constructing such using \texttt{FROM}, but its main operators all work on graph pattern matching. Aside from the SPARQL \texttt{UPDATE} functionality, the only operators to do so are \texttt{FROM} and \texttt{FROM NAMED}. We propose to extend the \texttt{FROM} functionality of {\SPARQL} and call this extension {\ESPARQL} to establish actual sets of facts based on which triple pattern matching and downstream operations (\texttt{FILTER}, \texttt{AGGREGATE}, etc.) may be applied.

{\ESPARQL} allows constructing graphs in \texttt{FROM} clauses filtering triples with patterns that can include epistemic conditions. For example, the following {\ESPARQL} query asks for all full deities according to Christianity.
\begin{lstlisting}[language=SPARQL]
SELECT ?deity
FROM BELIEF <Christianity>
WHERE { ?deity a <FullDeity> }
\end{lstlisting}
\noindent
Intuitively, the clause \texttt{FROM BELIEF <Christianity>} generates a graph whose triples are marked as either \emph{true}, \emph{false}, \emph{unknown}, or \emph{conflicting}. These truth values represent the \texttt{<Christianity>} beliefs. The graph that results from extracting the Pope beliefs is then used for matching the triple pattern \texttt{?deity a <FullDeity>}, and returning the bindings for variable \texttt{?deity}. The answers for Jesus, his father God, and his mother Maria will be annotated as conflicted, true, and false, respectively, and individuals that are not mentioned will be annotated as unknown.

This paper makes the following contributions:
\begin{enumerate}
\item
  We describe use case requirements for an epistemic query language (Section~\ref{sec:UseCase}).
\item
  We provide a semantics for {\RDFStar} and {\SPARQLStar} as a $\sK$-annotated algebra, which operates over functions from {\SPARQLStar} solution mappings to $\sK$ elements (Section~\ref{sec:annotated-sparqlstar}). This algebra extends the Geerts et al.~$\sK$-annotated algebra~\cite{DBLP:journals/jacm/GeertsUKFC16}, which is used as the foundation for how-provenance in {\SPARQL}.
\item
  We propose a query language, {\ESPARQL}, an extension to the $\sK$-annotated {\SPARQLStar} where the abstract semiring structure $\sK$ is substituted by the concrete semirings in the $\FOUR$ structure (Section~\ref{sec:eSPARQL}).
\item
  We show that {\ESPARQL} can express the use case queries, and present an user interface for the {\ESPARQL} algebra (Section~\ref{sec:UseCaseDiscussion}).
\item
  We show that {\SPARQL} queries can lead to query results that cannot be encoded with finite expressions using the support of the solution (Section~\ref{sec:Safe-eSPARQL}). We propose an approach to deal with these infinite results.
\end{enumerate}


%% file: 3-UseCase.tex
\section{Use Case Requirements}%
\label{sec:UseCase}

We have explored a broad range of queries that an epistemic query language for {\RDFStar} should be able to answer. We distil four query examples that this language must allow formulating easily. They require (i) to query for people's beliefs, (ii) to query for aggregated or integrated belief states, (iii) to query for subjects whose beliefs coincide, conflict, are compatible or ignorant of other belief sets, (iv) allow for nested belief states.

\newcommand{\usecase}[1]{\paragraph*{\textnormal{\textbf{#1}}}}

\usecase{Use Case U1 (Query for people's beliefs)}

We should be able to query people's beliefs regarding a given statement, and the query should return the bindings to the possible variables in the statement and the respective four-valued logic truth values for the variable bindings. The expression for these queries should be simple (i.e., do not include any algebraic operators like $\aAND$, $\aUNION$, $\aFILTER$, or $\aSELECT$).

\emph{Example: Return the full deities $\varX$ according to Pope Damasus the First.}
All returned individuals should be annotated with a belief value that indicates the belief of the Pope regarding the individual. For example, Jesus must be annotated as \emph{true} if the Pope believes Jesus is a full deity, \emph{false} if the Pope believes Jesus is not a full deity, \emph{unknown} if there is no information about the belief of the Pope regarding the nature of Jesus, and \emph{conflicted} if either the Pope believes this statement is conflicted, or the belief is both \emph{true} and \emph{false}.

\usecase{Use Case U2 (Query for combined belief values)}
We should be able to combine people's beliefs to represent the belief of a group, or operations between different people's beliefs.

\emph{Example: Return the beliefs of Christian people regarding the full deities.}
For example, Jesus must be annotated with true if all Christians think Jesus is a \emph{full} deity, \emph{false} if all Christians think Jesus is not a full deity, \emph{conflicted} if they have contradictory opinions, and unknown if no Christian has an opinion regarding the nature of Jesus.

\usecase{Use Case U3 (Query for people whose beliefs coincide, conflict, or are compatible)}
Given a set of statements, we should be able to query people who coincide (they all have the same belief for all the statement), conflict (some people believe a statement is true, whereas others believe it is false), or they are compatible (if a person believes a fact is true, the others can believe it is true or unknown).

\emph{Example: Return all people whose opinion is conflicted with Pope Damasus the First in at least one statement.}
For example, since Pope Damasus the First believes that Jesus is a full deity, and Arius believes it is not, then Arius should be returned (i.e., annotated with true). Since Russell is agnostic regarding an opinion of Pope Damasus the First, Russell is not in conflict with the Pope, so Russell must be annotated with false.

\usecase{Use Case U4 (Query nested belief states)}
We should be able to query about what people believe other people believe.

\emph{Example: Return all the people $x$ such that there exists a person $y$ that believes that $x$ believes that Zeus is not a full deity.}
In our example in the introduction, we must annotate Pope Damasus the First with true, because Russell believes the pope believes Zeus is not a full deity. Other individuals should be annotated with unknown.


%% file: 2-Preliminaries.tex
\section{Preliminaries}%
\label{sec:Preliminaries}

\paragraph{\textbf{The $\FOUR$ structure.}}
A \emph{partially ordered set} (or \emph{poset}) is a set taken together with a partial order on it. Formally, a partially ordered set is defined as a pair $(A, \leq)$, where $A$ is called the \emph{ground set} of the poset and $\leq$ is the \emph{partial order} of the poset.

A poset $(A,\leq)$ defines a lattice if for every two elements $a,b \in A$, there is a unique maximum element $c \in A$ such that $c \leq a$ and $c \leq b$, called the \emph{meet} of $a$ and $b$, and there is a unique minimum element $c \in A$ such that $a \leq c$ and $b \leq c$, called the \emph{join} of $a$ and $b$.

We write $\FOUR$ to denote the set $\{ \false, \true, \unknown, \conflicting \}$, and we use \emph{$\FOUR$ elements} to refer to the elements in this set. We define the posets $(\FOUR, \leqTruth)$ and $(\FOUR, \leqInf)$ as the minimum posets satisfying the inequalities $\LeqTruth{\false}{\unknown}$, $\LeqTruth{\false}{\conflicting}$, $\LeqTruth{\unknown}{\true}$, $\LeqTruth{\conflicting}{\true}$, $\LeqInf{\unknown}{\false}$, $\LeqInf{\unknown}{\true}$, $\LeqInf{\false}{\conflicting}$, and $\LeqInf{\true}{\conflicting}$. These two posets with ground set $\FOUR$, define the bilattice depicted in Figure~\ref{fig:four-bilattice}. The axis can be interpreted as the orderings regarding truth and information, and the $\FOUR$ elements are the respective elements of the 4-valued logic \emph{false} ($\false$), \emph{true} ($\true$), \emph{unknown} ($\unknown$), and \emph{conflicting} ($\conflicting$). We use \emph{$\FOUR$ operators} to refer to the operators $\TruthMeet$, $\TruthJoin$, $\InfMeet$, and $\InfJoin$ (see Figure~\ref{fig:four-bilattice}), and \emph{$\FOUR$ structure} to refer to the algebraic structure that consists of the $\FOUR$ operators over the $\FOUR$ set. Given a $\FOUR$ operator $\circ$ and a set $A = \{\alpha_1,\dots,\alpha_n\} \subseteq \FOUR$, we write $\Id{\circ}$ to denote the identity element for an operator $\circ$ in the $\FOUR$ structure (i.e., $\Id{\TruthMeet} = \true$, $\Id{\TruthJoin} = \false$, $\Id{\InfMeet} = \conflicting$, and $\Id{\InfJoin} = \unknown$), we write $\Absorbing{\circ}$ to denote the absorbing element for a $\FOUR$ operator $\circ$, (i.e., $\Absorbing{\TruthMeet} = \false$, $\Absorbing{\TruthJoin} = \true$, $\Absorbing{\InfMeet} = \unknown$, and $\Absorbing{\InfJoin} = \conflicting$), and we write $\Reduce{\circ}{A}$ for the $\FOUR$ element $\Id{\circ} \circ \alpha_1 \circ \dots \circ \alpha_n$.

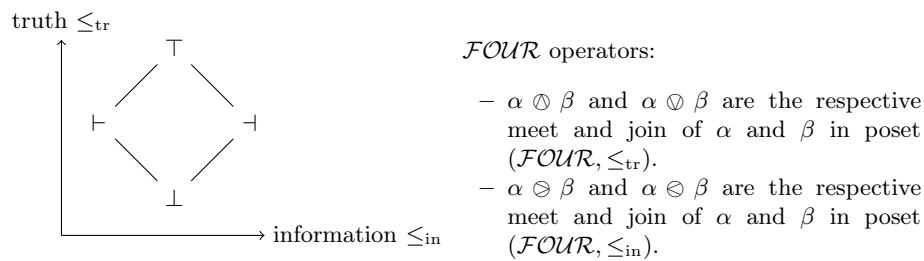
\begin{figure}[t]
  \centering
  \begin{tikzpicture}
    \node(true) at (0,1) {$\true$};
    \node(false) at (0,-1) {$\false$};
    \node(conf) at (1,0) {$\conflicting$};
    \node(unk) at (-1,0) {$\unknown$};
    \path[draw] (true)--(unk)--(false)--(conf)--(true);
    \path[draw,->] (-1.5,-1.5)--(1.2,-1.5) node [right] {information $\leqInf$};
    \path[draw,->] (-1.5,-1.5)--(-1.5,1.1) node [above] {truth $\leqTruth$};
  \end{tikzpicture}
  \hfill
  \begin{minipage}[b]{0.5\hsize}
    $\FOUR$ operators:
    \begin{itemize}
    \item
      $\alpha \TruthMeet \beta$ and $\alpha \TruthJoin \beta$ are the respective meet and join of $\alpha$ and $\beta$ in poset $(\FOUR, \leqTruth)$.
    \item
      $\alpha \InfMeet \beta$ and $\alpha \InfJoin \beta$ are the respective meet and join of $\alpha$ and $\beta$ in poset $(\FOUR, \leqInf)$.
    \end{itemize}
\end{minipage}
  \caption{Bilattice of the $\FOUR$-structure.}
  \label{fig:four-bilattice}
\end{figure}

A \emph{monoid} $\sM$ is an algebraic structure $(M, \SumM, \ZeroM)$ where $M$ is a non-empty set, $\SumM$ is a closed associative operation on set $M$, and $\ZeroM \in M$ is the identity element of operation $\SumM$ (i.e., $\ZeroM \SumM a = a$ and $a \SumM \ZeroM = a$, for every $a \in M$). The monoid $\sM$ is said to be \emph{commutative} if operation $\SumM$ is commutative.

A \emph{semiring} $\sK$ is an algebraic structure $(K, \SumK, \ProdK, \ZeroK, \OneK)$ where $(K, \SumK, \ZeroK)$ and $(K, \ProdK, \OneK)$ are monoids, and $\ProdK$ distributes over $\SumK$ (i.e., $a \ProdK (b \SumK c) = (a \ProdK b) \SumK (a \ProdK c)$ and $(b \SumK c) \ProdK a = (b \ProdK a) \SumK (c \ProdK a)$). The semiring $\sK$ is said to be \emph{commutative} if operations $\SumK$ and $\ProdK$ are commutative.
For example, the algebraic structures $(\FOUR, \TruthJoin, \TruthMeet, \false, \true)$, $(\FOUR, \InfJoin, \InfMeet, \unknown, \conflicting)$, $(\{\false, \true\}, \TruthJoin, \TruthMeet, \false, \true)$, and $(\mathbb{N}, +, \cdot, 0, 1)$ are commutative semirings.

We will write $\FOURTruth$ and $\FOURInf$ for the semirings $(\FOUR, \TruthJoin, \TruthMeet, \false, \true)$ and $(\FOUR, \InfJoin, \InfMeet, \unknown, \conflicting)$, respectively.

%% file: 4-AnnotatedSPARQLStar.tex
\section{$\sK$-annotated {\SPARQLStar}}%
\label{sec:annotated-sparqlstar}

In this section, we extend a fragment of the Geerts et al.~\cite{DBLP:journals/jacm/GeertsUKFC16} $\sK$-annotated {\SPARQL} algebra to support {\RDFStar} and {\SPARQLStar} (following the notions from the working draft~\cite{rdfstar}).
We consider a fragment without the non-monotonic operator \texttt{MINUS} because the non-monotonic operators go beyond the semiring structure of each of the two $\FOUR$ lattices.

We assume two countable pairwise disjoint sets $\SetIRIs$ and $\SetVars$, called the set of \emph{ IRIs} and the set of \emph{variables}.
We call the elements of $\SetIRIs^3$ \emph{RDF triples}\footnote{In RDF, triples can also have other element types, namely blank nodes and literals. For simplicity, we do not consider them, since they do not change our results.}.
Given an RDF triple $(s,p,o) \in \SetIRIs^3$, we say that $s$ is the \emph{subject}, $p$ is the \emph{predicate}, and $o$ is the \emph{object} of the triple.
According to the working draft~\cite{rdfstar}, \emph{{\SPARQLStar} triples} and \emph{{\SPARQLStar} triple patterns} are defined recursively as follows. Given two {\SPARQLStar} triple patterns $T_1$ and $T_2$, and an {\RDF} triple pattern $\triple{S}{P}{O}$, the triples $\triple{S}{P}{O}$, $\triple{T_1}{P}{O}$, $\triple{S}{P}{T_2}$, and $\triple{T_1}{P}{T_2}$ are {\SPARQLStar} triple patterns. A {\SPARQLStar} triple pattern without variables is an \emph{{\RDFStar} triple} (and $\mathcal{T}$ is the set of all {\RDFStar} triples), and an \emph{{\RDFStar} graph} is a set of {\RDFStar} triples.

A \emph{{\SPARQLStar} solution mapping} (or a \emph{mapping}) is a partial function $\mu : \SetVars \to \SetIRIs \cup \mathcal{T}$ with finite domain $\dom(\mu)$. Two mappings $\mu$ and $\mu'$ are \emph{compatible} if $\mu(\varX) = \mu'(\varX)$ for every variable $\varX \in \dom(\mu) \cap \dom(\mu')$. If mappings $\mu$ and $\mu'$ are compatible, $\mu \cup \mu'$ is the mapping with domain $\dom(\mu) \cup \dom(\mu')$ that is compatible with $\mu$ and $\mu'$. Given a set of variables $W \subset \SetVars$, we write $\Mappings{W}$ to denote the set of mappings $\{\mu \st \dom(\mu) = W\}$.

We next extend {\RDFStar} and {\SPARQLStar} to support annotations over a semiring $\sK$.
Given a commutative semiring $\sK = (K, \SumK, \ProdK, 0, 1)$, an \emph{{\RDFStar} $\sK$-graph} is a function $\AGraph : \mathcal{T} \to K$ that maps every {\RDFStar} triple to an element of the semiring $\sK$.

\begin{remark}
  Notice that $\sK$ denotes an arbitrary semiring, whereas $\FOURTruth$ and $\FOURInf$ are two concrete semirings. 
  Thus, the $\sK$-annotated {\SPARQL} algebra by Geerts et al.~\cite{DBLP:journals/jacm/GeertsUKFC16} is defined in abstract, 
  whereas the $\FOURTruth$-annotated and the $\FOURInf$-annotated {\SPARQLStar} algebras are concrete instances of the abstract algebra proposed by Geerts et al.
\end{remark}

\begin{definition}[Syntax of $\sK$-annotated {\SPARQLStar}]
  We next present a recursive definition for the syntax of \emph{$\sK$-annotated {\SPARQLStar}} queries $Q$ and a finite set of variables, called the \emph{in-scope variables}, denoted $\inScope(Q)$.
  \begin{enumerate}
  \item
    A {\SPARQLStar} triple pattern $T$ is a $\sK$-annotated {\SPARQLStar} query, whose in-scope variables are the variables occurring in triple pattern $T$.
  \item
    Given a $\sK$-annotated {\SPARQLStar} query $Q$ and a set of variables $W \subseteq \inScope(Q)$, the expression $\qSELECT{W}{Q}$ is a $\sK$-annotated {\SPARQLStar} query with in-scope variables~$W$.
  \item
    Given a $\sK$-annotated {\SPARQLStar} query $Q$ and a formula $\varphi$, the expression $\qFILTER{Q}{\varphi}$ is a $\sK$-annotated {\SPARQLStar} query with in-scope variables $\inScope(Q)$.
  \item
    Given two $\sK$-annotated {\SPARQLStar} queries $Q_1$ and $Q_2$, the expression $\qAND{Q_1}{Q_2}$ is a $\sK$-annotated {\SPARQLStar} query with in-scope variables $\inScope(Q_1) \cup \inScope(Q_2)$.
  \item
    Given two $\sK$-annotated {\SPARQLStar} queries $Q_1$ and $Q_2$ with the same set $W$ of in-scope variables, the expression $\qUNION{Q_1}{Q_2}$ is a $\sK$-annotated {\SPARQLStar} query with in-scope variables $W$.
  \end{enumerate}
\end{definition}

\begin{definition}[Semantics of $\sK$-annotated {\SPARQLStar}]
  Given a semiring $\sK = (K, \SumK, \ProdK, \ZeroK, \OneK)$, the result of evaluating a $\sK$-annotated {\SPARQLStar} query $Q$ over a $\sK$-graph $G$, is a function $\sparqlEval{Q}{\AGraph} : \Mappings{\inScope(Q)} \to K$, called \emph{$\sK$-relation}, defined recursively as follows:
  \begin{enumerate}
  \item
    Given a {\SPARQLStar} triple pattern $T$, if $\dom(\mu) = \inScope(T)$ then $\sparqlEval{Q}{\AGraph}(\mu)$ is $\AGraph(t)$, where $t$ is the {\RDFStar} triple resulting from replacing every variable $\varX$ occurring in $T$ with $\mu(\varX)$; otherwise, $\sparqlEval{Q}{\AGraph}(\mu) = \ZeroK$.
  \item
    For the recursive cases, the semantics is the following:
    \begin{enumerate}
    \item $\sparqlEval{\qSELECT{W}{Q}}{\AGraph}(\mu) =
      \sum_{\mu': \mu'|_W=\mu}\sparqlEval{Q}{\AGraph}(\mu')$,
    \item
      $\sparqlEval{\qFILTER{Q}{\varphi}}{\AGraph}(\mu) =
      \sparqlEval{Q}{\AGraph}(\mu) \times_{\sK} k_{\mu\models\varphi}$,
    \item
      $\sparqlEval{\qUNION{Q_1}{Q_2}}{\AGraph}(\mu) =
      \sparqlEval{Q_1}{\AGraph}(\mu) +_{\sK} \sparqlEval{Q_2}{\AGraph}(\mu)$,
    \item
      $\sparqlEval{\qAND{Q_1}{Q_2}}{\AGraph}(\mu) =
      \sparqlEval{Q_1}{\AGraph}(\mu|_{\inScope(Q_1)}) \times_{\sK}
      \sparqlEval{Q_2}{\AGraph}(\mu|_{\inScope(Q_2)})$,
    \end{enumerate}
    where $\sum$ denotes sums using $\SumK$, $\mu'|_{W}$ is the projection of mapping $\mu'$ on the variables in $W$, and $k_{\mu \models \varphi}$ is $\OneK$ if replacing the variables in $\varphi$ according to $\mu$ leads to a true formula, and $k_{\mu \models \varphi}$ is $\ZeroK$, otherwise.
  \end{enumerate}
\end{definition}

\begin{table}
  \caption{{\RDFStar} triples in the running example graph $G = \{t_1,\dots,t_8\}$. The {\RDFStar} triples $t_9$, $t_{10}$, and $t_{11}$ are not stated in $G$ but occur nested in the triples in $G$.}
  \label{table:running-example}
  \centering
  \begin{tabular}{p{0.05\hsize}p{0.922\hsize}}
    \toprule
    Id & Triple \\
    \midrule
    $t_1$ & $\ToBeTrue{\Pope}{\triple{\Jesus}{\type}{\FullDeity}}$, \\
    $t_2$ & $\ToBeFalse{\Arius}{\triple{\Jesus}{\type}{\FullDeity}}$, \\
    $t_3$ & $\ToBeConflicted{\Christianity}{\triple{\Jesus}{\type}{\FullDeity}}$, \\
    $t_4$ & $\ToBeUnknown{\Russell}{\triple{\Jesus}{\type}{\FullDeity}}$ \\
    $t_5$ & $\ToBeTrue{\Russell}{\ToBeTrue{\Pope}{\triple{\Jesus}{\type}{\FullDeity}}}$ \\
    $t_6$ & $\ToBeTrue{\Russell}{\ToBeFalse{\Pope}{\triple{\Zeus}{\type}{\FullDeity}}}$ \\
    $t_7$ & $\triple{\Pope}{\type}{\Christian}$ \\
    $t_8$ & $\triple{\Arius}{\type}{\Christian}$ \\
    \midrule
    $t_9$ & $\triple{\Jesus}{\type}{\FullDeity}$, \\
    $t_{10}$ & $\ToBeFalse{\Pope}{\triple{\Zeus}{\type}{\FullDeity}}$ \\
    $t_{11}$ & $\triple{\Zeus}{\type}{\FullDeity}$ \\
    \bottomrule
  \end{tabular}
\end{table}

\begin{example}[Running Example]%
  \label{ex:running-example}
  Let $G = \{t_1 \mapsto \true, \dots, t_8 \mapsto \true, * \mapsto \unknown \}$ be the {\RDFStar} $\FOURInf$-graph such that $G(t_i) = \true$ for every {\RDFStar} triple $t_i$, with $1 \leq i \leq 8$, listed in Table~\ref{table:running-example}, and $G(t') = \unknown$ for the rest of all possible {\RDFStar} triples $t'$.
  Then, $\sparqlEval{\ToBeFalse{\varX}{\triple{\varY}{\type}{\FullDeity}}}{G}$ returns the $\sK$-relation $R = \{ \{ \varX \mapsto \Arius, \varY \mapsto \Jesus \} \mapsto \true, * \mapsto \unknown \}$. Intuitively, this query returns all people who believe it to be false that somebody is a full deity.
\end{example}

\paragraph{\textbf{Support of functions.}}
So far, we have described the $\FOUR$ structure, and the extended Geerts et al.~\cite{DBLP:journals/jacm/GeertsUKFC16} $\sK$-annotated algebra.
In principle, the elements of this algebra, $\sK$-graphs and $\sK$-relations, can have no finite representation.
One of the concepts used to characterize $\sK$-graphs and $\sK$-relations with a finite representation is the support of a function, which we describe next.

Given two functions $f: A \to B$ and $g: A \to B$, and a binary operation $\diamond$ closed in $B$, $f \diamond g$ is the function such that $(f \diamond g)(x) = f(x) \diamond g(x)$.
If $\diamond$ has an identity element $\Id{\diamond}$, the \emph{support} of $f$ regarding operation $\diamond$ is the set $\support_{\diamond}(f) = \{ a \in A \st f(a) \neq \Id{\diamond} \}$. If the support of $f$ is a finite set $\{a_1, \dots, a_n\}$, we encode $f$ with the finite mapping $\{ a_1 \mapsto f(a_1), \dots, a_n \mapsto f(a_n), * \mapsto \Id{\diamond} \}$, where $*$ is a wildcard for the (possibly infinite) remaining mappings that are mapped to the identity of the operation.

\begin{remark}
  Intuitively, the support represents all values that are assigned to \emph{non-zero} values. In the case of a semiring $\sK = (K,\SumK, \ProdK, \ZeroK, \OneK)$ the zero value is $\ZeroK$, whereas in the $\FOUR$ structure we can consider two zero values, namely $\false$ and $\unknown$, depending on which semiring we are considering. These two values are the corresponding bottoms of the two lattices in the $\FOUR$ structure. Intuitively, the zero values in $\FOUR$ represent the default value for statements that are not included in a knowledge base, according to either the \emph{closed-world assumption} or the \emph{open-world assumption}.
\end{remark}

\begin{remark}
  According to the specification, an {\RDFStar} graph is a finite set of {\RDFStar} triples $\{T_1,\dots,T_n\}$ and under the open-world assumption (which is the standard for {\RDFStar}) a triple $T$ is true if $T \in G$ and unknown if $T \notin G$. Hence, the {\RDFStar} graph $G$ can be interpreted as the $\FOURInf$-graph that is represented by the finite mapping $\{T_1 \mapsto \true, \dots, T_n \mapsto \true, * \mapsto \unknown \}$. Furthermore, the evaluation of a $\sK$-annotated {\SPARQLStar} query over a $\sK$-graph with finite support always returns a $\sK$-relation with finite support. This finite property, guarantees that the algebra can be implemented.
\end{remark}

\begin{proposition}
  Given a semiring $\sK$, for every {\RDFStar} $\sK$-graph $G$ and every {\SPARQLStar} query $Q$, the $\sK$-relation $\sparqlEval{Q}{G}$ has finite support.
\end{proposition}

\begin{proof}
  This can be shown by induction on the structure of the query.
\end{proof}

\begin{remark}
  The $\FOURInf$-annotated {\SPARQLStar} algebra allows us to query $\FOURInf$-graphs, but does not provide the means to easily formulate the epistemic queries described in Section~\ref{sec:UseCase}. This motivates the proposal of the query language we describe in the next section.
\end{remark}


%% file: 5-eSPARQL.tex
\section{Epistemic SPARQL}%
\label{sec:eSPARQL}

This section presents the syntax and semantics of Epistemic SPARQL ({\ESPARQL}), the language designed for the use cases described in Section~\ref{sec:UseCase}.

So far, we have defined the notion of {\SPARQLStar} $\sK$-graphs, where $\sK$ is a semiring. However, a $\sK$-graph is an abstract notion, since no concrete semiring is provided. If concrete semirings like $\FOURTruth$ and $\FOURInf$ are considered, then we have concrete notions as $\FOURTruth$-graphs and $\FOURInf$-graphs (see Example~\ref{ex:running-example}). In what follows, a $\FOUR$-graph will be a function that associates {\RDFStar} triples to elements in set $\FOUR$, without choosing one of the two semirings ($\FOURTruth$ or $\FOURInf$) defined by the $\FOUR$ structure.

A key characteristic of $\FOUR$-graphs is that they can contain epistemic metadata encoded using four predicates: $\believesToBeTrue$, $\believesToBeFalse$, $\believesToBeUnknown$, and $\believesToBeConflicted$. Each of these predicates encodes a belief. For example, triple $t_1 = \ToBeTrue{\Pope}{t_9}$ from Table~\ref{table:running-example} encodes that $\Pope$ believes statement $t_9$ to be true.

A basic operation over a $\FOUR$-graph $G$ is thus extracting beliefs for each of these predicates.
We next present \emph{belief queries}, which are expressions to extract a $\FOUR$-graph $G'$ representing the belief of one or more people according to the information of an input $\FOUR$-graph $G$. An atomic belief query $\AtomicBeliefQuery{\Pope}{\true}{\unknown}$ represents all statements $\Pope$ believe to be $\true$, assuming that the rest are annotated with the state $\unknown$. Compound belief queries are generated by combining atomic statements with the $\FOUR$ operators.

\begin{definition}[Syntax of a Belief Query]
  A \emph{belief query} $E$ is recursively defined as follows:
  \begin{enumerate}
  \item
    Given an element $u \in \SetIRIs \cup \SetVars$, and two $\FOUR$ elements $\alpha$ and $\beta$, the triple $\AtomicBeliefQuery{u}{\alpha}{\beta}$ is an \emph{atomic} belief query.
  \item
    Given two belief queries $E_1$ and $E_2$, and a $\FOUR$-operator $\circ$, the expression $(E_1 \circ E_2)$ is a \emph{compound} belief query.
  \end{enumerate}
  We write $\var(E)$ to denote the set of variables occurring in a belief query $E$.
  A belief query with no variables is said to be \emph{ground}.
  We write $(u, \circ)$ to denote the belief query $[u, \true, \Id{\circ}] \circ [u, \false, \Id{\circ}] \circ [u, \unknown, \Id{\circ}] \circ [u, \conflicting, \Id{\circ}]$.
\end{definition}

\begin{definition}[Semantics of a Ground Belief Query]
  Let $\circ$ be a $\FOUR$-operator. The semantics of a belief query $E$ over a $\FOUR$-graph $G$ is given by the $\FOUR$-graph, denoted $\GraphContext{G}{E}$, defined recursively as follows:
  \begin{enumerate}
  \item
    Given an element $a \in \SetIRIs$ and a $\FOUR$ operation $\circ$,
    \begin{align*}
      \GraphContext{G}{\AtomicBeliefQuery{a}{\true}{\beta}}(t)
      &= \left\{
        \begin{array}{ll}
          \true & \text{ if } G(\ToBeTrue{a}{t}) \in \{\true, \conflicting\},\\
          \beta & \text{ otherwise}.
        \end{array}
        \right.\\
      \GraphContext{G}{\AtomicBeliefQuery{a}{\false}{\beta}}(t)
      &= \left\{
        \begin{array}{ll}
          \false & \text{ if } G(\ToBeFalse{a}{t}) \in \{\true, \conflicting\},\\
          \beta & \text{ otherwise}.
        \end{array}
        \right.\\
      \GraphContext{G}{\AtomicBeliefQuery{a}{\unknown}{\beta}}(t)
      &= \left\{
        \begin{array}{ll}
          \unknown & \text{ if } G(\ToBeUnknown{a}{t}) \in \{\true, \conflicting\},\\
          \beta & \text{ otherwise}.
        \end{array}
        \right.\\
      \GraphContext{G}{\AtomicBeliefQuery{a}{\conflicting}{\beta}}(t)
      &= \left\{
        \begin{array}{ll}
          \conflicting & \text{ if } G(\ToBeConflicted{a}{t}) \in \{\true, \conflicting\},\\
          \beta & \text{ otherwise}.
        \end{array}
        \right.
    \end{align*}
  \item
    $\GraphContext{G}{(E_1 \circ E_2)}(t) =
    \GraphContext{G}{E_1}(t) \circ \GraphContext{G}{E_2}(t)$.
  \end{enumerate}
\end{definition}

The reader may wonder why we find believes on statements that are annotated as $\true$ or $\conflicting$ instead of only $\true$. The reason is that, regarding the information dimension, a statement annotated with $\conflicting$ is equivalent to a statement annotated twice, as $\true$ and $\false$. Hence, the $\conflicting$ annotated statement includes the $\true$ annotated statement.

\begin{example}
  Let $G$ be the $\FOUR$-graph described in Example~\ref{ex:running-example}. Then,
  \begin{enumerate}
  \item
    $\GraphContext{G}{\AtomicBeliefQuery{\Pope}{\true}{\unknown}} = \{t_9 \mapsto \true, * \mapsto \unknown\}$. Intuitively, the resulting graph has all the statements that are true according to $\Pope$ and all unstated statements to be unknown.
  \item
    $\GraphContext{G}{(\Pope,\InfJoin)} = \{t_9 \mapsto \true, * \mapsto \unknown\}$,
    $\GraphContext{G}{(\Arius,\InfJoin)} = \{t_9 \mapsto \false, * \mapsto \unknown\}$, and
    $\GraphContext{G}{(\Russell,\InfJoin)} = \{t_1 \mapsto \true, t_{10} \mapsto \false, * \mapsto \unknown\}$. Intuitively, these are the respective beliefs of $\Pope$, $\Arius$, and $\Russell$.
  \item
    $\GraphContext{G}{((\Pope,\InfJoin) \InfJoin (\Arius,\InfJoin))} = \{t_9 \mapsto \conflicting, * \mapsto \unknown\}$. Intuitively, $\Pope$ and $\Arius$ are conflicted regarding the nature of Jesus deity.
  \item 
    $\GraphContext{G}{((\Pope,\InfJoin) \InfJoin (\Russell,\InfJoin))} = \{t_1 \mapsto \true, t_9 \mapsto \true, t_{10} \mapsto \true, * \mapsto \unknown\}$. Intuitively, $\Pope$ and $\Russell$ are not conflicted regarding the nature of Jesus deity, so the most informative beliefs is assigned to triple $t_9$.
  \end{enumerate}
\end{example}

\begin{definition}[Syntax of {\ESPARQL} Filter Formulas]
  \emph{{\ESPARQL} filter formulas} are defined recursively as follows:
  \begin{enumerate}
  \item
    Given two variables $\varX, \varY \in \SetVars$ and two IRIs $a, b \in \SetIRIs$, the expressions $\varX = \varY$, $\varX = a$, $a = b$, and $\qBound{\varX}$ are atomic {\ESPARQL} filter formulas.
  \item
    Given a $\FOUR$ value $\alpha$, $\qState{\alpha}$ is an atomic {\ESPARQL} filter formula.
  \item
    Given two {\ESPARQL} filter formulas $\varphi$ and $\psi$, the expressions $\neg\varphi$, $(\varphi \land \psi)$ and $(\varphi \lor \psi)$ are compound {\ESPARQL} filter formulas.
  \end{enumerate}
\end{definition}

\begin{definition}[Semantics of {\ESPARQL} Filter Formulas]
  Given a $\FOUR$-relation $R$, a mapping $\mu$ in the domain of $R$, and a {\ESPARQL} filter formula $\varphi$, the truth value of $\varphi$ on $\mu$ is defined recursively as follows:
  \begin{enumerate}
  \item
    If $\varphi$ has the form $u = v$ then $\mu(\varphi) = \Error$ if there is a variable $\varX$ in $\varphi$ such that $\varX \notin \dom(\mu)$, $\mu(\varphi) = \True$ if replacing every variable $\varX$ in $\varphi$ with $\mu(\varphi)$ leads to a true entity, and $\dom(\mu)$, otherwise, $\mu(\varphi) = \False$.
  \item
    If $\varphi$ has the form $\qBound{\varX}$ and $\varX \in \dom(\mu)$, $\mu(\varphi) = \True$; otherwise $\mu(\varphi) = \False$.
  \item
    If $\varphi$ has the form $\qState{\alpha}$, $\mu(\varphi) = \True$ if $R(\mu) = \alpha$; otherwise, $\mu(\varphi) = \False$.
  \item
    If $\varphi$ is a compound {\ESPARQL} filter formula then $\varphi$ is evaluated with the standard SPARQL three-valued logic semantics of the logical connectives $\neg$, $\land$, and $\lor$ for the values $\True$, $\False$, and $\Error$ (see~\cite{DBLP:journals/tods/PerezAG09}).
  \end{enumerate}
  We define the relation $\mu \models_R \varphi$ to be true if $\mu(\varphi) = \True$ for $R$.
\end{definition}

\noindent
Intuitively, {\ESPARQL} filter formulas extend {\SPARQL} filter formulas with the ability to check the state of the current mapping.

\begin{definition}[Syntax of {\ESPARQL}]
  The \emph{syntax of {\ESPARQL} queries} and their \emph{in-scope} variables is defined recursively as follows:
  \begin{enumerate}
  \item
    An RDF-star triple pattern $T$ is an {\ESPARQL} query whose in-scope variables are the variables occurring in $T$.
  \item
    Given two {\ESPARQL} queries $Q_1$ and $Q_2$ and a $\FOUR$ operation $\circ \in \{ \InfMeet, \TruthMeet\}$, the expression $(Q_1 \circ Q_2)$ is an {\ESPARQL} query whose in-scope variables are\break $\inScope(Q_1) \cup \inScope(Q_2)$.
  \item
    Given two {\ESPARQL} queries $Q_1$ and $Q_2$ with the same set $W$ of in-scope variables and a $\FOUR$ operation $\circ \in \{\InfJoin, \TruthJoin\}$, the expression $(Q_1 \circ Q_2)$ is an {\ESPARQL} query whose in-scope variables are $W$.
  \item
    Given a {\ESPARQL} query $Q$, a $\FOUR$ operation $\circ$, and a filter formula $\varphi$, the expression $\qFILTERe{\circ}{Q}{\varphi}$ is an {\ESPARQL} query with in-scope variables $\inScope(Q)$.
  \item
    Given an {\ESPARQL} query $Q$, a $\FOUR$ operator $\circ$, and a finite set of variables $W \subset \inScope(Q)$, the expression $\qSELECTe{\circ}{W}{Q}$ is an {\ESPARQL} query with in-scope variables $W$.
  \item
    Given an {\ESPARQL} query $Q$, a filter formula $\varphi$, and two $\FOUR$ elements $\alpha$ and $\beta$, the expression $\qSTATEUPDATE{Q}{\varphi}{\alpha}{\beta}$ is an {\ESPARQL} query with in-scope variables $\inScope(Q)$.
  \item
    Given a belief query $E$, and an {\ESPARQL} query $Q$ such that the in-scope variables of $Q$ do not appear in $E$, the expression $\qBELIEFe{E}{Q}$ is an {\ESPARQL} query with in-scope variables $\var(E) \cup \inScope(Q)$.
  \end{enumerate}
\end{definition}

\begin{definition}[Semantics of {\ESPARQL}]\label{def:esparql-semantics}
  The result of evaluating an \emph{{\ESPARQL} query} query $Q$ on a $\FOUR$-graph $G$ is a function $\kset{Q}{G} : \Omega_{\inScope(Q)} \to \FOUR$, called $\FOUR$-relation, defined recursively as follows:
  \begin{enumerate}
  \item
    Given an RDF-star triple pattern $T$ and a mapping $\mu \in \Omega_{\inScope(Q)}$, $\kset{T}{G}(\mu)$ is $G(t)$, where $t$ is the {\RDFStar} triple pattern resulting from substituting every variable $\varX \in \inScope(Q)$ with $\mu(\varX)$.
  \item
    $\kset{\qANDTruth{Q_1}{Q_2}}{G}(\mu) =
      \kset{Q_1}{G}(\mu|_{\inScope(Q_1)}) \TruthMeet \kset{Q_2}{G}(\mu|_{\inScope(Q_1)})$.
  \item
    $\kset{\qANDInf{Q_1}{Q_2}}{G}(\mu) =
      \kset{Q_1}{G}(\mu|_{\inScope(Q_1)}) \InfMeet \kset{Q_2}{G}(\mu|_{\inScope(Q_1)})$.
  \item
    $\kset{\qUNIONTruth{Q_1}{Q_2}}{G}(\mu) = \kset{Q_1}{G}(\mu) \TruthJoin \kset{Q_2}{G}(\mu)$.
  \item
    $\kset{\qUNIONInf{Q_1}{Q_2}}{G}(\mu) = \kset{Q_1}{G}(\mu) \InfJoin \kset{Q_2}{G}(\mu)$.
  \item
    $\kset{\qFILTERe{\circ}{Q_1}{\varphi}}{G}(\mu) = \kset{Q_1}{G}(\mu) \circ \alpha_{\mu \models \varphi}$, where $\alpha_{\mu \models \varphi} = \Id{\circ}$ if formula $\mu \models_R \varphi$ in the $\FOUR$-relation $R = \kset{Q_1}{G}$, and $\alpha_{\mu \models \varphi} = \Absorbing{\circ}$, otherwise.
  \item
    $\kset{\qSELECTe{\circ}{W}{Q_1}}{G}(\mu) = \Reduce{\circ}{%
      \{ \kset{Q_1}{G}(\mu_1) \st \mu_1|_W = \mu \}
    }$.
  \item
    $\kset{\qSTATEUPDATE{Q_1}{\varphi}{\alpha}{\beta}}{G}(\mu) = \gamma$, where $\gamma = \alpha$ if $\mu \models_{\kset{Q_1}{G}} \varphi$, and $\gamma = \beta$, otherwise.
  \item
    $\kset{\qBELIEFe{E}{Q}}{G}(\mu) = \kset{Q}{\GraphContext{G}{E'}}(\mu|_{\inScope(Q)})$, where $E'$ is the belief query resulting from replacing every variable $\varX$ in $E$ with $\mu(\varX)$.
  \end{enumerate}
\end{definition}

\noindent
Intuitively, each of the {\SPARQL} operators corresponds to two {\ESPARQL} operators (e.g., $\aAND$ corresponds to $\aANDTruth$ and $\aANDInf$, and $\aUNION$ corresponds to $\aUNIONTruth$ and $\aUNIONInf$), and each of these operators are evaluated according to the Geerts et al.~\cite{DBLP:journals/jacm/GeertsUKFC16} $\sK$-annotated SPARQL algebra by choosing $\sK$ to be one of the two $\FOUR$ semirings, namely $\FOURTruth$ and $\FOURInf$. Except for the operator $\aBELIEFe$, which changes the context graph where the query is evaluated.

\begin{example}\label{ex:esparql-semantics}
    Let $G =\{t_1, \dots, t_8\}$ be the annotated graph mentioned in Table~\ref{table:running-example}. Consider the query 
    \(Q = \qBELIEFe{(\varX, \InfJoin)}{\triple{\varY}{\type}{\FullDeity}}\).
    According to Definition~\ref{def:esparql-semantics}, the result of query $Q$ is a $\FOUR$-relation with mappings $\mu$ whose domain includes the variables $\varX$ and $\varY$. Given such a mapping $\mu$, 
    $\kset{Q}{G}(\mu) = \kset{\triple{\varY}{\type}{\FullDeity}}{\GraphContext{G}{(\mu(\varX), \InfJoin)}}(\mu|_{\{\varY\}})$.
    For example, if $\mu(\varX) = \Pope$ and $\mu(\varY) = \Jesus$, then
    \[
      \kset{Q}{G}(\mu) \begin{aligned}[t]
        &= \kset{\triple{\varY}{\type}{\FullDeity}}{\GraphContext{G}{(\Pope, \InfJoin)}}(\mu|_{\{\varY\}}) \\
        &= \GraphContext{G}{(\Pope, \InfJoin)}(\triple{\Jesus}{\type}{\FullDeity}) \\
        &= \true.
      \end{aligned}
    \]
    It is not difficult to see that if variable $\varX$ had been bound to an individual whose beliefs are not encoded in the graph, then $\kset{Q}{G}(\mu) = \bot$ because $\GraphContext{G}{(\Jesus, \InfJoin)}$ would have hold no information (i.e., it would have annotated all triples with $\bot$).
\end{example}


%% file: 6-UseCaseDiscussion.tex
\section{Use Case Requirements Discussion}%
\label{sec:UseCaseDiscussion}

In this section, we show a {\ESPARQL} query for each of the use cases U1--U4. Since the notation of {\ESPARQL} queries in an algebraic format is not suitable for end-users, we additionally present how this query can be written as an extension of the user {\SPARQL} syntax.

\subsection*{Use Case U1 Query}
\[
    \qBELIEFe{(\Pope, \InfJoin)}{\triple{\varDeity}{\type}{\FullDeity}}
\]
The clause $\qBELIEFe{(\Pope, \InfJoin)}{\cdot}$ generates the graph $G'$ consisting of all beliefs of $\Pope$. The states of duplicate statements are aggregated with operation $\InfJoin$, and statements that are not mentioned are defined to have state $\unknown$.
Then, the nested query $\triple{\varDeity}{\type}{\FullDeity}$ returns a $\FOUR$-relation $R$ whose domain consists of all mappings $\Omega|_{\variable{deity}}$. This domain includes all mappings of the form $\mu_u = \{\variable{deity} \mapsto u\}$ where $u \in \SetIRIs$. For each $u \in \SetIRIs$, $R(\mu) = G'(\triple{u}{\type}{\FullDeity})$. Since the only belief of $\Pope$ is indicated in triple $t_1$, we conclude that $R(\mu)$ is $\true$ if $u$ is $\Jesus$, and $R(\mu)$ is $\unknown$, otherwise (see Example~\ref{def:esparql-semantics}).

Observe that the query includes the operator $\InfJoin$. This means that the operations are done over the information lattice. In the user {\ESPARQL} syntax, this is indicated in a simple way by introducing the modifier \texttt{INFO} to the \texttt{SELECT} clause. The next listing shows how the graph $G'$ is specified using the \texttt{FROM BELIEF} clause.

\begin{lstlisting}[language=SPARQL]
SELECT INFO ?deity
FROM BELIEF <PopeDI>
WHERE { ?deity a <FullDeity> }
\end{lstlisting}

\subsection*{Use Case U2 Query}
\[
  (\aSELECTe{\InfJoin}~{\varDeity}\;
    \begin{aligned}[t]
      (& (\triple{\varX}{\type}{\Christian} \aSTATEUPDATE \qState{\true}~\conflicting~\unknown) \;
      \aANDInf \\
      & \qBELIEFe{(\varX, \InfJoin)}{\triple{\varDeity}{\type}{\FullDeity}}))
    \end{aligned}
\]
The internal clause $Q_2 = \qBELIEFe{(\varX, \InfJoin)}{\cdot}$ works similarly to the previous example, but the ground belief queries are constructed over variable $\varX$. Thus, we need to bind this variable to entities whose belief are encoded in the graph (see Example~\ref{def:esparql-semantics}). There are then four possible bindings for variable $\varX$, namely $\Pope$, $\Arius$, $\Christianity$, and $\Russell$. The answers $\mu_2$ of the query $Q_2$ will be joined with the answers $\mu_1$ of the nested query $Q_1 = (\triple{\varX}{\type}{\Christian} \aSTATEUPDATE \qState{\true}~\conflicting~\unknown)$. This imposes a further restriction of the binding for variable $\varX$ (remember that $\InfMeet$ acts as \texttt{AND}). Only for $\Pope$ and $\Arius$, the state $\kset{Q_1 \InfJoin Q_2}{G}(\mu_1 \cup \mu_2)$ is not $\bot$. Thus, the $\qSELECTe{\InfJoin}{\varDeity}{\cdot}$ aggregates the beliefs of these two instances of variable $\varX$ with operation $\InfJoin$. If $\varY$ is bound to $\Jesus$, then for $\Pope$ the mapping $\mu_2$ is annotated as $\true$, whereas for $\Arius$ it is annotated as $\false$. Hence, the belief of Christians regarding the nature of $\Jesus$ is $\conflicting$.

In user {\ESPARQL} syntax, this query can be expressed using \texttt{FROM BELIEF} with a variable from within a nested \texttt{SELECT INFO} query.

\begin{lstlisting}[language=SPARQL]
SELECT INFO ?deity
WHERE {
  ?x a <Christian> .
  MAP IF (STATE IS TRUE) TO CONFLICTED ELSE UNKNOWN .
  {
    SELECT INFO ?deity
    FROM BELIEF ?x
    WHERE { ?deity a <FullDeity> }
  }
}
\end{lstlisting}

\subsection*{Use Case U3 Query}
\[
  (\qBELIEFe{%
    ((\Pope,\InfJoin) \InfJoin (\varX,\InfJoin))}{%
    \triple{\variable{s}}{\variable{p}}{\variable{o}}}
  \aSTATEUPDATE \qState{\conflicting}~\true~\false)
\]
We first consider the believe query $((\Pope,\InfJoin) \InfJoin (\varX,\InfJoin))$, which joins the beliefs of $\Pope$ with all other $\varX$ using $\InfJoin$. 
The $\FOUR$-relation constructed by $\triple{\variable{s}}{\variable{p}}{\variable{o}}$ is unconstrained; thus, we obtain the join of the information lattice over all triples.
Before projecting to $\varX$ via $\aSELECTe{\TruthJoin}$ we use $\aSTATEUPDATE$ to mark all mappings that include any conflict, which are aggregated by the $\aSELECTe{\TruthJoin}$.
Finally, we map $\conflicting$ to $\true$ and everything else to $\false$.

In the user {\ESPARQL} syntax, we use \texttt{SELECT} without any modifiers, which defaults to the truth lattice.

\begin{lstlisting}[language=SPARQL]
SELECT ?x
WHERE {
  {   
    SELECT INFO ?x
    FROM BELIEF <PopeDI> ?x
    WHERE { ?s ?p ?o }    
  }
  MAP IF (STATE IS CONFLICTED) TO TRUE ELSE FALSE
}
\end{lstlisting}

\subsection*{Use Case U4 Query}
\[
  \qSELECTe{\InfJoin}{\varX}{(\qBELIEFe{(\varY, \InfJoin)}{\qBELIEFe{(\varX, \InfJoin)}{\triple{\Zeus}{\type}{\FullDeity}}}}
\]
The nested query $\qBELIEFe{(\varY,\TruthJoin)}{\cdot}$ generates the $\FOUR$-annotated graph of all beliefs by $\varY$ aggregating them on the information lattice. Similarly, the nested query $\qBELIEFe{(\varX,\TruthJoin)}{\cdot}$ generates the $\FOUR$-annotated graph of all beliefs by $\varX$ according the beliefs of $\varY$. Then, the innermost clause $\triple{\Zeus}{\type}{\FullDeity}$ obtains a $\FOUR$-relation $R$ with mappings $\mu$ whose in-scope variables are $\varY$ and $\varX$. For example, $R(\{\varY \mapsto \Russell, \varX \mapsto \Pope\}) = \true$ because according to the beliefs of $\Russell$, $\Pope$ beliefs that $\triple{\Zeus}{\type}{\FullDeity}$. On the other hand, $R(\{\varY \mapsto \Arius, \varX \mapsto \Pope\}) = \unknown$ because the graph does not contain information about the beliefs of $\Arius$ regarding the beliefs of $\Pope$.
Finally, the clause $\qSELECTe{\InfJoin}{\varX}{\cdot}$ aggregates the mappings on such a relation $R$ for each instance of the variable $\varX$.

In the user {\ESPARQL} syntax, we use a nested $\aSELECT$ query for each level of nesting in the beliefs.

\begin{lstlisting}[language=SPARQL]
SELECT INFO ?x
WHERE {
  {
    SELECT INFO *
    FROM BELIEF ?y
    WHERE {
      SELECT INFO *
      FROM BELIEF ?x
      WHERE { <Zeus> a <FullDeity> }
    }
  }
}
\end{lstlisting}


%% file: 7-Safe_eSPARQL.tex
\section{Finitely Supported eSPARQL}%
\label{sec:Safe-eSPARQL}

To implement {\ESPARQL}, the query results must be finitely encoded. As we have already shown, we can encode an infinite function with a finite mapping that has a finite support by encoding the non-zero states only.

\begin{definition}
  An {\ESPARQL} query $Q$ is said to be \emph{finitely supported} if, for every $\FOUR$-graph with finite support, there exists an element $\alpha \in \FOUR$, called a \emph{zero} for $Q$ and $G$, such that the set $\{ \mu \mid \kset{Q}{G}(\mu) \neq \alpha \}$ is finite.
\end{definition}

To know if {\ESPARQL} queries are finitely supported, consider two $\FOUR$-relations $R_1$ and $R_2$ with finite support that result of evaluating the respective queries $Q_1$ and $Q_2$. That is, there are two $\FOUR$ operations $\circ$ and $\diamond$, such that the sets $\support_\circ(R_1)$ and $\support_\diamond(R_2)$ are finite. Given a $\FOUR$ operation $\bullet$, does $R_3 = R_1 \bullet R_2$ have a finite support? To answer this question, lets start figuring what is needed for $R_3$ to have a finite support for each operator combination of the operators $\circ$, $\diamond$, and $\bullet$.

The first observation is that $R_1$ and $R_2$ have infinitely many mappings whose states are respectively $\Id{\circ}$ and $\Id{\diamond}$. If $\bullet$ is a join operator (i.e., $\TruthJoin$ or $\InfJoin$), then we will have infinitely many mappings $\mu$ such that $R_3(\mu) = \Id{\circ} \bullet \Id{\diamond}$. Since there is exactly one $\FOUR$-operator, namely $*$ such that $\Id{*} = \Id{\circ} \bullet \Id{\diamond}$, we want to know if $\support_*(R_3)$ is finite. In this case, the answer is straightforward. If we have a mapping $\mu'$ such that $R_3(\mu) \neq \Id{*}$ then it must happen that $R_1(\mu) \neq \Id{\circ}$ or $R_2(\mu) \neq \Id{\diamond}$. Since $R_1$ and $R_2$ have finite support, there are finitely many mappings $\mu$ satisfying that condition. Thus, $\support_*(R_3)$ is finite. Hence, for the two operations corresponding to the two lattice join (i.e., the generalizations of the {\SPARQL} \texttt{UNION}), two $\FOUR$-relations with finite support result in a $\FOUR$-relation with finite support.

We next show that if $\bullet$ is a meet operator (i.e., $\TruthMeet$ or $\InfMeet$), then the resulting $\FOUR$-relation can have no finite support. To this end, consider the query $Q_1 \InfMeet Q_2$ where $Q_1$ and $Q_2$ are two queries with respective in-scope variables $\varX$ and $\varY$. Given a $\FOUR$-annotated graph $G$, let $R_1$ be $\kset{Q_1}{G}$ and $R_2$ be $\kset{Q_2}{G}$, $\support_\TruthMeet(R_2)$ be finite, $R_1(\{\varX \mapsto a\}) = \true$, $R_1(\{\varX \mapsto b\}) = \false$, and assume that there are infinitely many values $c$ such that $R_2(\{\varY \mapsto c\}) = \true$. Then,
\[
  \begin{aligned}[t]
    &\kset{Q_1 \InfMeet Q_2}{G}(\{\varX \mapsto a, \varY \mapsto c\}) = \true,\\[-3pt]
    &\kset{Q_1 \InfMeet Q_2}{G}(\{\varX \mapsto b, \varY \mapsto c\}) = \unknown.  
  \end{aligned}
\]
Since we can take infinitely many values of $c$ to produce these two different states, we conclude that the answer of query $Q_1 \InfMeet Q_2$ has no finite support.

\begin{proposition}
  Every fragment of {\ESPARQL} that includes the operations $\TruthMeet$ or $\InfMeet$ can include queries that are not finitely supported.
\end{proposition}

This negative result is not necessarily an impediment for implementing {\ESPARQL}. In general, since we are interested in statements about the individuals who appear in a knowledge graph, it suffices to consider mappings that range to the active domain of the graph (i.e., individuals that occur in triples in the support of the graph). Since this subset of $\SetIRIs$ is finite, using the active domain will lead to a query language whose results can be finitely encoded.


%% file: 8-RelatedWork.tex
\section{Related Work}
\label{sec:RelatedWork}


There are many works on four-valued logics~\cite{DBLP:conf/ideas/GrahneM18,DBLP:journals/jphil/Restall95,DBLP:journals/ai/Patel-Schneider89a}, but no one of them considers them for the semantics and query evaluation in SPARQL. Other works annotate SPARQL answers with lattice elements~\cite{DBLP:journals/pvldb/HernandezGH21,DBLP:conf/www/AsmaHGFFH24}, but they do not provide a mean to operate with sets of statements encoding beliefs. Arnout et al.~\cite{DBLP:journals/pvldb/ArnaoutRWP21} consider knowledge graphs with negative facts. However, they do not consider conflicted statements as we did. Works on distributed knowledge contexts are all assuming that knowledge was represented using different ontologies but with the same epistemological status \cite{DBLP:conf/semweb/BouquetGHSS03,DBLP:journals/ai/GhidiniS17,DBLP:journals/ws/GrauPS06}. Some of them used SPARQL queries as mappings \cite{DBLP:conf/www/SchenkS08}, but did not address epistemological status as a key concern. Schenk et al.~\cite{DBLP:conf/semweb/Schenk08} studies the semantics of trust and caching in the Semantic Web considering the $\FOUR$ structure, however did not consider the problem of querying.


%% file: 9-Conclusion.tex
\section{Conclusions and Future Work}
\label{sec:Conclusion}

We presented {\ESPARQL}, a novel approach that allows for the description of epistemic information using {\RDFStar} and formulating epistemic queries using a query language which extends {\SPARQLStar}. This query language is based on the concrete $\FOUR$ bilattice, but we expect to generalize it to include more general bilattices.

Future work could include the study of different ways to implement {\ESPARQL}. The most direct way is to build on top of a standard {\SPARQLStar} engine. Indeed, the functionality of \texttt{FROM BELIEF} clauses to generate a new graph representing people's beliefs can be implemented with {\SPARQLStar} CONSTRUCT queries, which use aggregate operations to compute states of duplicated statements on a set of beliefs. Then, {\SPARQLStar} SELECT queries can be executed on top of the results from these CONSTRUCT queries. Other implementations include: (i) the whole rewriting of a {\ESPARQL} query as a single {\SPARQLStar} SELECT query, and (ii) using specialized indexes and algorithms to implement this query language.
